%% file: 0.main.tex
\documentclass[10.1pt, conference, letterpaper]{IEEEtran}

\usepackage{booktabs} % For formal tables
\usepackage{ifthen}
\usepackage{epsfig, times}
\usepackage{amssymb,amsmath,amsfonts,amsthm}
\usepackage{setspace}
\usepackage{array}
\usepackage{balance}
\usepackage{graphicx}
\usepackage{changes}
\usepackage{tabularx}
\usepackage{color}
\usepackage{bbm}
\usepackage{color,soul, colortbl} % Added colortbl
\usepackage{adjustbox} % Added
\usepackage{siunitx} % Added
\usepackage{url} 
\usepackage{listings}
\usepackage{comment}
\usepackage{caption}
\usepackage{subcaption}
\definecolor{Gray}{gray}{0.9} % Added
\definecolor{darkgray}{rgb}{0.66, 0.66, 0.66}% Added

\providecommand{\ecameraready}{si}
\providecommand{\stampacommenti}{si} 

\long\def\symbolfootnote[#1]#2{\begingroup%
\def\thefootnote{\fnsymbol{footnote}}\footnote[#1]{#2}\endgroup}
\usepackage{mathtools}
\usepackage[hidelinks]{hyperref} 

\hypersetup{
  colorlinks   = false, %Colours links instead of ugly boxes
  urlcolor     = black, %Colour for external hyperlinks
  linkcolor    = black, %Colour of internal links
  citecolor   = black %Colour of citations
}

\newtheorem{theorem}{Theorem}
\newtheorem{lemma}{Lemma}

\newtheorem{remark}{Remark}[section]

\theoremstyle{definition}
\newtheorem{definition}{Definition}

\newif\ifcre
\ifthenelse{\equal{\ecameraready}{no}}{\cretrue}{\crefalse}
\ifcre
  \providecommand{\forex}[1]{}
   \providecommand{\forcr}[1]{#1}
\else
   \providecommand{\forex}[1]{#1}
   \providecommand{\forcr}[1]{}
\fi

\newif\ifcom
\ifthenelse{\equal{\stampacommenti}{si}}{\comtrue}{\comfalse}

\ifcom
   \providecommand{\com}[1]{{\color{red}\emph{#1}}}
\else
   \providecommand{\com}[1]{}
\fi

\newcommand{\fref}[1]{Fig.~\ref{#1}}
\newcommand{\sref}[1]{Section~\ref{#1}}
\newcommand{\cref}[1]{Chapter~\ref{#1}}

\newcommand{\eref}[1]{Equation~(\ref{#1})}

\newcommand{\lref}[1]{Lemma~\ref{#1}}
\newcommand{\thref}[1]{Theorem~\ref{#1}}

\def\bearn{\begin{eqnarray*}}
\def\eearn{\end{eqnarray*}}
\def\bear{\begin{eqnarray}}
\def\eear{\end{eqnarray}}
\def\barr{\begin{array}}
\def\earr{\end{array}}
\def\bmat{\left(\begin{array}}
\def\emat{\end{array}\right)}

\def\bp{\small\begin{list}{}\item \textbf{Proof:}}
\def\ep{\end{list} \qed \normalsize}

\def\bp{\hspace{4mm}\emph{Proof:}}

\def\bd{\begin{definition}}
\def\ed{\end{definition}}
\def\bt{\begin{theorem}}
\def\et{\end{theorem}}
\def\be{\begin{center}\begin{equation}}
\def\ee{\end{equation}\end{center}}
\def\bc{\begin{corollary}}
\def\ec{\end{corollary}}
\def\bl{\begin{lemma}}
\def\el{\end{lemma}}
\def\br{\begin{remark}}
\def\er{\end{remark}}

\makeatletter
\newcommand{\distas}[1]{\mathbin{\overset{#1}{\kern\z@\sim}}}%
\newsavebox{\mybox}\newsavebox{\mysim}
\newcommand{\distras}[1]{%
  \savebox{\mybox}{\hbox{\kern3pt$\scriptstyle#1$\kern3pt}}%
  \savebox{\mysim}{\hbox{$\sim$}}%
  \mathbin{\overset{#1}{\kern\z@\resizebox{\wd\mybox}{\ht\mysim}{$\sim$}}}%
}
\makeatother
\newtheorem{myproblem}{Problem}
\newtheorem{prop}{Proposition}

\usepackage{dsfont}
\usepackage{graphicx}
\usepackage{caption}
\usepackage{algorithmic}
\usepackage{algorithm}

\usepackage[shortlabels]{enumitem}
\usepackage{ulem}
\usepackage{physics}  %% for norm and abs

\usepackage{endnotes}

\usepackage{nccmath}
\usepackage{endnotes}

\let\footnote=\endnote
\usepackage{mathtools}

\usepackage[noadjust]{cite}
\usepackage{enumitem}
\setlist{nolistsep}
\setlist[itemize]{leftmargin=*}
\newtheorem{corollary}{Corollary}[theorem]

 \usepackage{acronym}
 \newacro{FC}{Floating Computing}
 \newacro{FG}{Floating Gossip}
 \newacro{ML}{Machine Learning}
 \newacro{RZ}{Replication Zone}
 
\begin{document}

\title{On limit performance ("learning capacity") of Gossip Learning in dynamic settings,  or location-based collaborative (fully) distributed model training in dynamic settings}

\title{On the Limit Performance of Floating Gossip}

% \title{The Learning Capacity of Floating Gossip}

\author{
\IEEEauthorblockN{
Gianluca Rizzo,\IEEEauthorrefmark{1}\IEEEauthorrefmark{3}
Noelia Perez Palma,\IEEEauthorrefmark{2}
Marco Ajmone Marsan,\IEEEauthorrefmark{2}
and Vincenzo Mancuso\IEEEauthorrefmark{2}
}
\IEEEauthorblockA{gianluca.rizzo@hevs.ch,
noelia.perez@imdea.org, ajmone@polito.it, vincenzo.mancuso@imdea.org}
%\IEEEauthorblockA{\IEEEauthorrefmark{1} HES SO Valais, Switzerland\;\;
%\IEEEauthorrefmark{2} University of Bern, Switzerland\;\;
%\IEEEauthorrefmark{3}University of Luxembourg, Luxembourg\;\; 
%\IEEEauthorrefmark{4}Luxembourg Institute of Science and Technology (LIST), Luxembourg}
%
\IEEEauthorblockA{
\IEEEauthorrefmark{1} HES SO Valais\;\;
\IEEEauthorrefmark{2} Institute IMDEA Networks\;\;
\IEEEauthorrefmark{3} Università di Foggia\;\;
%\IEEEauthorrefmark{4} Politecnico di Torino
}
}

%\renewcommand{\shortauthors}{B. Trovato et. al.}

\maketitle
\input{1.abstract}
\input{2.introduction}

\input{3.related_work}

\input{4.system_model_new}
\input{5.model}
\input{5.model_capacity}
\input{5.num_evaluation}
\input{6.conclusions}
\input{7.appendix}
\forcr{\clearpage}
\bibliographystyle{IEEEtran}
\bibliography{IEEEabrv,0.main}
%\appendix

% \newpage
% xx\\
% \newpage
% \input{Noelias_numeval.tex}

%\theendnotes
%\input{10.draft} 
%old things
%\input{9.algorithm}
%\theendnotes
\end{document}

%% file: 1.abstract.tex
\begin{abstract} 
In this paper we investigate the limit performance of Floating Gossip, a new, fully distributed Gossip Learning scheme which relies on Floating Content to implement location-based probabilistic evolution of machine learning models in an infrastructure-less manner. 

We consider dynamic scenarios where continuous learning is necessary, and we adopt a mean field approach to investigate the limit performance of Floating Gossip in terms of amount of data that users can incorporate into their models, as a function of the main system parameters.
Different from existing approaches in which either communication or computing aspects of Gossip Learning are analyzed and optimized, our approach accounts for the compound impact of both aspects.
We validate our results through detailed simulations, proving good accuracy.
Our model shows that Floating Gossip can be very effective in implementing continuous training and update of machine learning models in a cooperative manner, based on opportunistic exchanges among moving users.
%their ML model instances, until the size of ML models become so large that the opportunistic exchange becomes problematic.
%\com{Floating Gossip can provide a very effective approach for building ML models in areas where a network infrastructure is not available (e.g., in a post-disaster area) or the exchange of raw data is not possible.}
%alternative title: On Computational Capacity of Infrastructure-less Opportunistic Mobile Networks (mobile Gossip Computing)
%\com{I would use "gossip learning" or "dynamic GL" instead of "FG".}
\end{abstract}

%% file: 2.introduction.tex
\section{Introduction}

With the standardization of the 5G NR (New Radio) sidelink \cite{garcia2021tutorial}, in addition to the LTE sidelink \cite{molina2017lte}, WiFi direct \cite{ieee1999wireless}, and Bluetooth 5.2 \cite{woolley2020bluetooth}, a number of options now exist for implementing wireless D2D (Device to Device) communications \cite{ansari20175g}.
 Such rich set of technologies is likely to foster the development of new services based on the direct exchange of information among users. This will facilitate
the implementation of context-aware location-based 
%opportunistic 
services such as mobile crowdsensing applications, or services like Gigwalk \cite{GigwalkW92:online}, Waze \cite{Drivingd30:online}, and Millionagents \cite{MillionA18:online}, to name a few, since users can share location and context through D2D automatically.

In particular, services that require the continuous acquisition of fresh information (we will refer to them with the generic term \emph{observations}) about some specific context, and the subsequent elaboration in each user equipment (UE) of a machine learning (ML) model are poised to become common. Such ML models will allow UEs to devise up-to-date strategies for the optimization of a utility function while performing a task.

In this paper we study the effectiveness of D2D-based opportunistic services to support privacy-preserving ML-based decision making in UEs that operate autonomously, without the need for a RAN infrastructure, without centralized support, and without exchanging any observation. Our main goal is to determine the maximum amount of information that a group of UEs can autonomously learn and incorporate in their ML models, without centralized coordination, only exploiting their own computation capacity combined with D2D communications and gossip learning (GL)~\cite{ormandi2013gossip}. The GL approach is beneficial both in terms of maintaining data out of  reach of indirect observers---hence avoiding tampering with and exposing private information---and in terms of scalability, because it only moves ML models based on observations, not the observations themselves. Note that a reduction of the volume of data results in simpler processing at mobile sensing platforms  \cite{wang3317689} thus making simple devices---e.g., mobile UEs---able to participate in model training processes.

%The problem we tackle is within the domains of crowdsensing and ML, but, to the best of our knowledge, no previous paper has tried to characterize the intrinsic limits to ML in a fully distributed and opportunistic environment like ours.

Several works have previously looked at components of the system we consider. For example, a number of works on Floating Content (FC) have characterized the opportunistic D2D aspect~\cite{floating, ali2013tr,manzo2017performance}. Other works on GL have tackled the ML aspects~\cite{LSZXHJ21,GLGS19,HEGEDUS2021109}.
We name the system that we propose and analyze Floating Gossip (FG), because of its roots in GL and in FC augmented with computing within UEs.
Our investigation is based on a mean field approach validated through detailed simulation experiments. Specifically, the main contributions of this work are the following.

\begin{itemize}
    \item We present Floating Gossip, a fully distributed Gossip Learning  scheme based on localized probabilistic storage of models in an infrastructure-less manner.
    \item We derive a mean field-based analytical approach to modeling the limit performance of FG, in terms of \textit{model freshness} (i.e., of average time elapsed since the collection of the most recent observation included in an ML model) and of \textit{learning capacity}, i.e., of the maximum amount of observations which can be incorporated in an ML model.
    %\textcolor{cyan}{\sout{Note that the model freshness coincides with the average Age of Information (AoI) when time of collection coincides with time of generation.}}
    \item We evaluate numerically our results, validating our assumptions against simulation under different configurations, showing the accuracy of our mean field approach, and characterizing the properties of FG limit performance as a function of the main system parameters. 
\end{itemize}
%Our results highlight the show that
%\item Model freshness, i.e., average time elapsed since the collection of the most recent observation included in a ML model. Note that the model freshness coincides with the average Age of Information (AoI) if time of collection coincides with time of generation.
%\item Learning capacity, i.e., the maximum amount of observations which can be incorporated in a ML model. Note that learning capacity depends on computing and storage capacity at UEs as well as efficiency of FC, and in particular on the average number of nodes whose model incorporates a given observation.
% \item We evaluate numerically our results, validating our assumptions against simulation, under different mobility models, showing the accuracy of our mean field approach, and characterizing the properties of FC storage capacity as a function of the main system parameters.
% \end{itemize}
%\com{Our results show that ...
%}
%\com{The rest of this paper is organized as follows ... }

%% file: 3.related_work.tex
\section{Related Work}
%\com{shorten a lot, possibly merge with intro, search and add missing citations- you can fish them from our FC capacity journal paper and, mainly, from Mina's GL journal. Limit to the max citation of our works to avoid violating double blind}

Our work is inspired by Gossip Learning (GL)~\cite{ormandi2013gossip}, a collaborative ML approach that is motivated mainly by the need to guarantee data privacy and scalability. 
%GL assumes that the raw data collected by each device is not shared with other devices. 
In GL, learning is achieved by combining  model training on the local dataset of each node (updated over time with new observations) with model sharing and merging  with other devices.
We use GL with the restriction that model sharing can only occur over D2D within the FC boundaries, which imposes specific performance constraints in addition to computation constraints. A similar system has been proposed and studied in terms of accuracy of generated models in~\cite{LSZXHJ21}. In that work, federated learning relies on opportunistic model exchanges like in our work, and the authors show that encounter-based learning can be effective. However, they assume that training data can be at least partially exchanged over D2D, thus exposing data to privacy risks. 

Giaretta and Girdzijauskas~\cite{GLGS19} consider a static network, and analyze the impact on the behavior of GL of data distribution, of network connectivity, and communication speed. Our work considers the same aspects, adding the peculiarities of a dynamic,  opportunistic communication environment.

%takes not well explored assumptions on those three aspects. 

First, while a fully distributed data model is often assumed, in which each device only owns a data point and trains the model once (see~\cite{HEGEDUS2021109} and references therein) --- only the authors of~\cite{GLGS19} propose to train over multiple points, and present a simulation study for its performance impact --- we are more general, because we allow not only devices to train multiple times on different data points, but also account for the possibility that multiple devices might observe the same event and therefore own non-unique data points. The latter has not been considered in other works so far. 

Second, network connectivity is typically taken as stable, with little or no  limitations imposed by physical distance and interconnection infrastructure, which is often fully comprised within a same data center or within a multi-thread computing architecture~\cite{HEGEDUS2021109,JDAVCS18, BLOT2019287}. Instead, network connectivity is a key factor, which can slow down the convergence of GL, as shown for the particular case of stochastic gradient descent algorithms in~\cite{pmlr-v139-chen21y}.
We consider only opportunistic communications of physically distinct and not co-deployed devices, with no infrastructure support at all, so that the number of neighbors that can exchange models is limited. Most important, we consider dynamic topologies with churn in UEs, hence timely and opportunistically exchange of models becomes fundamental.

Third, while many works that focus on networked devices  consider unrealistically high communication speed or neglect transmission time and connection establishment overheads~\cite{GJZYWC22}, or neglect communications delay almost completely~\cite{MinaGM21}, we account for the limitations of D2D connectivity and protocols, which consist not only in relatively low data rates, but also in the fact that a connection takes time to be established and a device might be busy exchanging data with a UE when a new communication opportunity arises, which is therefore missed. Hence, differently from existing approaches in which either communication or computing aspects of GL are analyzed and optimized, here we focus on the compound impact of both aspects. 

Eventually, we remark that there are no comprehensive analytical models to evaluate the performance of GL, and all reviewed papers show results based on simulations or real trace analysis. In contrast, we only resort to simulation to validate the complex model developed in this paper. The closest work, related to the analytical methodology used in our work is the seminal work of Chaintreau \emph{et al.} on gossiping as a tool to spread information~\cite{chaintreau2009age}, which shows how to use mean field analysis on the epidemic diffusion of messages. We also use mean field analysis, but we account for the specifics of learning and computing in the generation of the information to be shared, and the specificity of FC as a means to keep models available in a given area of interest.

%% file: 4.system_model_new.tex
\section{The Floating Gossip Scheme---System model}
\label{sec:FGscheme}

%\vm{Since FG is a new thing defined in this paper, I would put less emphasis on related work and move it to the end of the paper. In a new Section II, I would rather provide the reader with the definition of FG and of its functions. I would use three subsections: in the first we can put high-level info taken from the currently available intro, and at the beginning of Section III, in the second subsection we can be more formal and present the definitions currently given in Section III, while the third subsection can contain all our assumptions (which are now spread over  Sections III to V).}

\subsection{What's Floating Gossip?}
%\vm{high-level info}
We consider a dynamic environment where many events happen, such as a crowded theme park or a post-disaster area. 
%Over this environment, a number of actors (humans or machines, that we term \emph{nodes}) operate for a variable amount of time.
During their stay in the environment, a number of actors (humans or machines, that we term \emph{nodes})  collect information that helps them make decisions and/or perform some task. The individually collected information consist\textcolor{cyan}{s} in local viewpoints (termed \emph{observations}) of what is happening in the environment at a given time. By using observations, each node builds its own partial, incomplete representation of the environment (termed \emph{model}), and may therefore take potentially sub-optimal decisions. %Indeed, possessing a global (or at least, a less partial) model can be very valuable to improve decisions' quality.  
%For such an environment, a centralized scheme for observation collection and fusion in a model might raise privacy concerns, or not be viable when communication between nodes and the central infrastructure is not available (this typically happens in a post-disaster scenario). 
To go beyond such limitations, we consider that nodes periodically combine their models with the ones received through D2D opportunistic exchanges.
%exchange their models using D2D whenever they come in proximity of one another. Then, by combining their own models with the ones received through opportunistic exchanges, 
%they generate a better model of the state of the environment.
The combined model is a potentially better representation of the environment because it incorporates a larger number of observations, some older, some more recent. 
In general, the utility of a model is related to the number of incorporated observations, as well as to the time at which they have been collected. 
%Indeed, since the system state evolves over time, the more recent and complete is the model obtained by fusion, the better can be the ensuing decisions.

Furthermore, we consider that models are relevant only in specific geographical areas, and that observations become obsolete after a fixed amount of time. Therefore stale observations are discarded, and nodes moving outside of the region of interest drop the instances of the models they own. 

Our approach is thus similar to the one adopted in GL, but it is based on models and observations whose availability persists in the  region of interest thanks to an opportunistic floating content scheme~\cite{Ott2011}. 
%We term the resulting learning scheme Floating Gossip (FG). While we use this more general term, o
Our approach can be seen as a machine learning process, where each observation is incorporated into the model (through a \emph{training} algorithm), and models are exchanged and fused (through a \emph{merging} algorithm), as in GL. In the rest of this paper we will refer to this case, but we will not address any specific ML algorithm, so that our model is general enough to be applicable to several different ML techniques.
 
\subsection{Some definitions}

We consider a population of nodes that cooperate in constructing a multidimensional representation of the environment in which they operate. Each dimension of the representation is associated with a model.

For the sake of simplicity we describe the system operations in the case of one model, concerned with one dimension of the environment (e.g., the road interruptions in the disaster area), but later we will also look at the case of several models evolving in parallel. 

Nodes move within an area, a portion of which defines the model Replication Zone (RZ). When entering the RZ, nodes have a \emph{model structure} (also called \emph{default model}), not yet populated with fresh data (for example, a rescue team entering a disaster area has a map that does not yet contain recent information on the critical issues in the area). While in the RZ, nodes collect data (i.e., observations) that are used to integrate the model structure with a training operation. By so doing, each node generates and updates its own \emph{model instance}. 

Each node is equipped with a wireless transceiver, and a positioning system that gives its position in space.
We say two nodes are \textit {in contact} when they are able to directly exchange information via wireless communications.

When two nodes come in contact of one another within the RZ, they may exchange their model instances by means of D2D, provided that the nodes are not  already engaged in other D2D exchanges, or that their models do not incorporate identical sets of observations. 
Each node then produces a better informed version of its own model instance by merging it with the received one.

A model is thus a description of the environment according to one dimension, and can be used to support decisions related to such dimension. %For example, a model of the road system in a disaster area allows finding paths to reach a rescue location avoiding road interruptions. 
The local model instance of a node is generally different from that of other nodes. 
A model instance is trained with the node's own observations (which form the model \textit{local training set}), and can be merged with other nodes' model instances to incorporate their observations (the union of the local training sets of all merged models forms the model total training set). 

The model training operation is a transformation which takes as input a model instance, and a new set of one or more observations, and gives as output a new model instance, with a training set given by the union of the training set of the old model instance and the new set of observations. In what follows we assume that model training is performed one observation at a time, though our approach easily extends to the case in which all observations in the new set are incorporated as a batch in the model instance.

The \emph{merging} operation is a transformation which takes as input two 
%\textcolor{red}{JUST TWO OR ALSO MORE THAN TWO?} or more 
instances of a model, and gives as output a new model instance, whose training set is the union of the training sets of both of the input instances. It occurs after an exchange of models among nodes, if the newly received model instance(s) incorporate observations that were not present in the local model instance.

The process by which the new model is derived is not specified in our work. Typically, in the case of Artificial Neural Networks (ANNs), the coefficients of the model obtained through merging are derived as a weighted average of the coefficients of the merged model instances, though more complex approaches are possible. Such a definition of the merging operation seems to suggest that it produces a model instance with the same performance (e.g., in terms of accuracy) as a model instance obtained by training over the union of the two training sets. However, this is well known to be an approximation.

Every time a node collects a new observation, it uses the new observation to train its local instance of the model.
The \emph{age} of an observation is the time since its generation.
%As a result, in general, a fraction of the observations within a model instance training set are endogenous, i.e., incorporated in the instance through training by the node on which they reside, while the others are exogenous, that is, incorporated through merging. 
%Model exchanges take place whenever two nodes come in contact of each other, using D2D communications. 
%Moreover, each observation is associated to a single specific model among those evolving in the RZ, and it can thus be incorporated only in an instance of that model.
%\com{ it needs some connection with practical applications.} 
Finally, users exiting the RZ discard their local model instance, as well as all collected observations.

\subsection{Modeling assumptions} 
\label{ss:assumptions}
%\vm{all our assumptions (which were spread over  various sections)}
%To model the operation of FG, we introduce some assumptions, some of which only serve to make the model tractable, and which will be later validated via simulation. 
%First of all, we use a mean field approach~\cite{}, and hence characterize FG's average behavior in steady state in a 

We consider nodes moving over the plane  according to a stationary mobility model such that, at any time instant, nodes are uniformly distributed in space. We consider the RZ to be a finite region of the plane, of arbitrary shape and area $A$. 
With $\alpha$ we denote the mean rate at which nodes enter the RZ, (which coincides with the mean rate at which nodes exit the same area because of equilibrium). 

We assume that the environment is constantly changing, thus requiring a continuous process of observation collection and incorporation in the models. We assume that there exist $M$ models in total, and each node can hold a single instance of up to $W$ models, assuming that a realistic implementation of FG will impose such practical limitation. When the number of models $M$ is larger than $W$, we assume each node chooses a subset of models of interest, which is equivalent to subscribe to up to $W$ observation channels. For the remaining $M-W$ models, the node will not accept any instance, nor it will record any observation associated to them.
To simplify the notation, in this paper we derive analytical expressions that hold for $M\ge W$. However, substituting $\min(M,W)$ for $M$ would make the same expressions valid for the general case.

%Any approximate global model of the environment can only be obtained by sharing observations collected by multiple mobile nodes in a distributed way, which we assume is implemented via replication of local model instances built by individual nodes, rather than by exchanging observations directly, as is done in legacy GL systems. 
Differently from GL, in FG learning occurs in a fully distributed way, using opportunistic data exchange
%\textcolor{cyan}{\sout{ without the help of a network infrastructure. Thus, FG relies
%on FC~\cite{floating} to ensure a target probability of model persistence and availability. 
%In turn, this means that, to transfer models, we resort to}, 
through D2D protocols. We partially neglect the overhead of such protocols and retain two main aspects: $i)$ connection establishment takes time $t_0$, which we consider as a constant, and $ii)$ data transfer is not instantaneous, but depends on the quantity of data to exchange bidirectionally, since we consider the transfer rate to be fixed. We also assume that data exchanges occur pairwise, and it is not possible to create and manage groups with more than two connected nodes, which would be highly unstable due to mobility. We say that two nodes establishing a connection or exchanging models are \textit{busy}, so they cannot accept any request to connect with other nodes until they finish and immediately disconnect.   

These parameters are important to determine the success of a model transfer upon a contact between two nodes. Other important factors are the frequency of contacts observed by each node, $g$, and the duration of contacts, which is denoted by $t_c$. %and represents the time during which two mobile nodes are closer than a transmission range interval. %Having adopted a mean field model, 
The values of $g$ and $t_c$ are the same for all nodes. In particular, $t_c$ is a random variable with 
probability density function (pdf) denoted by $f(t_c)$. The mean time necessary for the transfer of a model instance between two nodes is instead denoted as $T_L$, which is a function of the capacity of the channel between two nodes in contact, as well as of the size of the model instances, which we assume to be equal to $L$ bits. The order in which model instances are transferred upon a contact is important, but since scheduling optimization is out of the scope of our investigation, we assume that models to be transferred are chosen at random among the available ones.  

Once exchanged, instances belonging to the same model are then merged through local computation, again without the help of an infrastructure. We assume that received model instances are not used and propagated until they have been fused with the local instance. This requires storage and computing resources at the nodes not only for incorporating observations into local model instances (training), but also to fuse different instances of the same model, when they are received from neighbors (merging). We assume that storage/compute resources are used as in a system with two FIFO queues, one for training and one for merging. We assume a single shared service with finite capacity, that  gives strict priority to merging tasks, with no preemption of training tasks. We do this because merging tasks move more information at once than training tasks, so they are more important to enable the derivation of better models. For simplicity, we assume that queues have no buffering limit, although in practice we will show that they need limited storage space. Training and merging times are assumed to be constant, and are indicated as $T_T$ and $T_M$, respectively.

Moreover, since observations are sequentially generated and model instances are asynchronously acquired, they are queued individually and then trained or merged one by one. However, an observation can be recorded simultaneously by $1 \le \Lambda \le W$ nodes (i.e., by the $\Lambda$ nodes interested in the model the observation is relevant for, and being in proximity of the point at which the observation can be gathered). Model instances resulting from training and merging are uniquely identified by the set of observations used to build the merged instances. Therefore, we also assume that two instances of the same model with a same (aggregate) training set are indistinguishable.
%, independently on whether the observations were locally collected or embedded thanks to merging with other model instances.
We assume however that relevant observations are all those generated within a maximum age, which we denote as $\tau_l$. That is, observations older than $\tau_l$ do not need to be incorporated into a model, and they are not counted as part of the training set. This makes sense in those scenarios in which the process which underlies the generation of the observations is not stationary over time.
%in which case training an ML model with old data might make it unfit for producing accurate inference. 
Finally, we remark that since the model size is finite, an instance of the model can only accommodate a limited number of observations. We therefore assume that if a new observation is added when the model instance is ``full'', the oldest observation is dropped from the training set. 

%% file: 5.model.tex
\section{Performance Analysis of Floating Gossip}
%\com{t: say we do not start from a completely empty system.}
The main goal of the FG scheme is to enable the elaboration of ``well-enough trained" models, and to share them among users within the RZ.
%Note that we only look at observations included in a model through training or merging (it would make little difference to include nodes which have the observation but did not include it in any model, as given that time is continuous, there can be at most one such node in the system, for each observation). \footnote{\com{Alternatives: deliver to users entering the ZOI  \textit{before entering the ZOI}. OR: to at least x percent of users within the ZOI or RZ.}}
 
%Let us assume that the arrival process of observation generation is ergodic.
In what follows, we consider the system at times $t$ in which the process of model sharing within the RZ is in steady state. 
%We assume each model to have a finite capacity, i.e. it is able to store a finite amount of observations. 
%When we incorporate a set of $n$ new observations in a model instance which is full (i.e. whose training set already incorporates the maximum amount of observations possible), the model discards $n$ of the oldest observations present in the training set. \com{!!!}

A first performance metric for our FG scheme is the \textit{observation incorporation rate}, i.e., the rate at which observations with a given age are incorporated in a model. 
\begin{definition} \textit{For a given model, a given age of observation $\tau$, and for a time $t$, the observation incorporation rate $R(t,\tau)$ at time $t$ is the average of the ratio between the amount of observations with age $\in (\tau,\tau +\Delta\tau)$ incorporated in an instance of the model at time $t$, and the time interval $\Delta\tau$, for $\Delta\tau\rightarrow 0$.}
\end{definition}
%
% Such a mean must be intended as a temporal mean, over the time $t$ at which the model instance (and the observation it incorporates) is observed. 
% \com{(vm: this is not clear; it does not seem to be a time average but a statistical average over observations with the same age.)}
%If we consider the system at the mean field limit, then such quantity is also the mean across all model instances present in the RZ. 
%When the transient of the process of model diffusion is exhausted, given that the intensity of the process of observation generation is constant over time, the incorporation rate does not depend on $t$, or on the location of observation generation.
Since we study the system at steady state, we assume that the observation incorporation rate does not depend on $t$, or on the location of observation generation. The observation incorporation rate is an indicator of the effectiveness and of the speed with which our FG scheme is able to store new observations in a model. %\com{It is a function of the likelihood of an observation to be incorporated by training in a model, of the speed at which observations diffuse within nodes in the RZ, and also of how long they manage to persist in the RZ even after that the nodes on which they were generated exit the RZ.} 

Another measure of the ability of the FG scheme to incorporate the most recent measured data is the \textit{model staleness}.
\begin{definition} \textit{The staleness $F_i$ of a model instance is the average time elapsed since the generation of the most recent observation included in the model instance. The staleness $F$ of a model is the average of $F_i$ over all instances of the model present in the RZ.}
\end{definition}
%The staleness $F$ of a model is the time elapsed since the generation of the most recent observation included in an instance of that model, averaged over time and across all instances of a model present in the RZ at a given time.}
%
%\com{I would prefer to use this definition: 
%}
Model staleness is a measure of how much ``up to date" on average an instance of a model is, in a given scenario. Note however that a staleness value of $\tau$ does not say \textit{how many observations} are included in a given interval $(\tau, \tau + \Delta \tau)$. Thus, for instance, it does not necessarily imply that enough observations are included on average in an instance of the model to allow the model to have a comprehensive picture of the status of the observed system at $\tau$. 
% Finally, we assume that those observations relevant for the performance of the application to which each model is associated are all those generated within a maximum age, which we denote as $\tau_l$. That is, we assume that observations older than $\tau_l$ do not need to be incorporated into a model, and they are not counted as part of the training set. This makes sense in those scenarios in which the process which underlies the generation of the data of the observations is not stationary over time, in which case training a ML model with old data might make it unfit for producing accurate inferences.
% \com{(vm: I fail to see why $\tau_l$ is needed in case of non-stationary systems...)\\
% GR: it is not needed, but we assume it derives from application requirements.\\
% (vm:...and what is the problem with having infinite memory in non-stationary systems. BTW, $\tau_l$ can be used also for introducing a finite memory in stationary systems. In any case, to me this is unnecessary complication for this work.\\
% GR: model capacity is a finite resource of the system, directly proportional to model size and thus to all the main performance parameters of this system.)}
%\vspace{-0.15in}
\subsection{Derivation of Observation Incorporation Rate and of Model staleness}
%In this section we present a set of analytical results for evaluating the incorporation rate of a GC scheme and of the mean model staleness as a function of the main system parameters.
%\com{(note that I am skipping here all of the MF detailed proofs, for the moment. Not so serious a limitation, given that the expressions coincide -at least so far- with a simplified version of those for which we already have proofs.)}
%\subsection{MF approach to model availability} 
%\com{MAM: We said previously that we seed observations, not models. GR: is this ok, now?}
%\com{If you assume that every node entering the RZ gets a a model (one with an empty training set, possibly with a set of default values for coefficients), then $a$ is the availability of model with nonempty training set.}.
%With $b(t,A)$ we denote the probability that a node within the RZ is busy exchanging models at time $t$.
\begin{definition} \textit{With reference to a given observation and the node on which it resides, the seeding probability $P$ is the probability that the observation is included in a model instance at the node by the time the node exits the RZ.}
\end{definition}
With $Y$ we denote the probability that at a node the local instance is not merged with a received instance because the training set of the former is a superset 
%(proper or improper) 
of the training set of the latter.
%However, in this work we are not interested in the transient of model diffusion, but in the value which $a_m(t,A)$, $m\in1,...,M$ take when transient is exhausted \com{(say why we are interested only in the steady state, talking about application scenarios.)}.
% \com{\begin{itemize}
%     \item if we talk about quasi stationarity, we should define it here and say when conditions of QS are satisfied.... basically part of what already discussed in the old paper on FC capacity. Or not mention it at all, as everyone else does in this field! maybe a footnote, with a pointer to the FC capacity paper would do.
%     \item say why we are interested only in the steady state, talking about application scenarios.
% \end{itemize}}. 
%To this end, in what follows we will focus on the mean values of $a(t,A)$ and $b(t,A)$ at the mean field limit, which we denote with $a$ and $b$, respectively.
% \footnote{\com{Do we want to define a success probability, e.g. related to a ZOI? Or stick with availability? I think we can do the latter, but say that we adapt to any scheme for model delivery, e.g. at the entrance/exit of ZOI or at exit of RZ. They are all simple functions of availability.}}
% Instead, we consider 
% For this reason, we are interested by the solutions of the system of equations \eref{eq:avail_model} and \eref{eq:b_model}, at the mean field limit, in the quasi-stationary regime. In what follows, we denote with $a$ and $b$ the steady state solutions of the considered ODE.
%The following results give their analytical expressions, as a function of the main system parameters. 
%
\begin{definition} 
\textit{We say that the system is in the substable regime when $P\approx 1$ and $Y \approx 0$.}
%\com{a possible way to set the following result is to say in the proof that we approximate by assuming $P=1$ and $Y=0$ (or that the respective effects they model are neglected).}
\end{definition}
The substable regime implies that the computing load is low enough for an observation to have very high chances to be incorporated in an instance by training. The condition $Y \approx 0$ instead is typically achieved when on average each model instance has a large training set, so that the likelihood that a merging between the local instance and an incoming model does not increase the training set of the local instance is low. In practice, the substable regime assumption is an approximation, which brings to accurate results when the above conditions are met. %\textbf{In the numerical section we will assess the impact of such assumption on the accuracy of our model.}

We now introduce three of the key parameters describing the dynamics of the FG system: the 
\textit{model availability}, the \textit{node busy probability} and the \textit{observation availability}.
\begin{definition} 
\textit{The availability of model $m$ at time $t$, denoted as $a_m(t,A)$, is the mean fraction of nodes possessing a non-default instance (i.e., an instance with a non-empty training set) of the $m$-th model, where $A$ is the RZ area. }
\end{definition}
%{\color{red}{E' NECESSARIA LA DIPENDENZA DA A?}}. 

In general,  
model availability takes different values for different models in a same RZ. 

\begin{definition} 
\textit{The probability that a node within a RZ with area $A$ is busy at time $t$, denoted with $b(t,A)$, is the probability that the node is not available for information exchange with other nodes in contact, at the given time instant.}
\end{definition}
%{\color{red}{We denote with $a$ and $b$, respectively, the mean field limits of $a_m(t,A)$ and $b(t,A)$.}}

The following result states that, $\forall m$, the steady state values of $a_m(t,A)$ and $b(t,A)$ depend neither on the specific trajectories (i.e., on the amount of initial seeders for each model) nor on the time at which each model has been seeded, but only on properties of the system once all transients are exhausted.
%Instead, the steady state value of model availability at the mean field limit is the same for any initial condition, and it depends only on the number of floating models, i.e. on properties of the system once all transients are exhausted.

\begin{lemma}\label{lemma:model_availability} For $t\rightarrow\infty$, in the substable regime the mean values of the variables $a_m(t,A)$, $\forall m$, and $b(t,A)$ at the mean field limit (denoted as $a$ and $b$, respectively) are given by the unique solution of the following fixed point problem in $a$:
\vspace{-0.2in}
\be \label{eq:steady_aa}
%\frac{b}{T_S}a(1-a)S+\frac{\lambda(1-a)}{AD} -\frac{\alpha}{AD}a=0,
a=0.5\left(H+\sqrt{H^2+\frac{4T_S(a)\lambda\Lambda}{bNS(a)w}}\right)
\ee
with 
\begin{align}
&H=1-\frac{T_S(a)(\alpha+\lambda\Lambda)}{bNS(a)w}&\nonumber\\
& b=K-\sqrt{K^2-1}&\nonumber\\
 & K=1+\frac{1}{4gT_S(a)}+\frac{\alpha}{2gN}&\nonumber\\
 %& \gamma= 2Ma[1-a\com{Y}]&\label{mean_exchangeables}\\
  %& \gamma= 2Ma\left[a(1-\com{Y})+\min\left(1-a,\frac{W}{Ma}-1\right)\right]&\label{mean_exchangeables}\\
  & \gamma= 2Mw^2a & \label{mean_exchangeables}\\
   &  S(a) = \int_{t_0}^{+\infty}\min{\left(1,\left \lfloor\frac{(t_c-t_0)}{T_{L}}\right \rfloor\frac{1}{\gamma}\right)}f(t_c)d t_c &\nonumber\\
 &T_S(a)= \int_{0}^{+\infty} \min{\left(t_c,\gamma T_{L}+t_0\right)}f(t_c)d t_c&\nonumber
 \end{align}
where $t_0$ is the transfer setup time, $N$ is the mean number of nodes in the RZ, and $w=\min\left(\frac{W}{M},1\right)$.
%, and it models the time required to prepare for content transfer once two nodes are in contact.
Independently from the initial condition, any trajectory of the system converges to the solution of \eref{eq:steady_aa}.
\end{lemma}
For the proof, please refer to \sref{sec:app_lemma:model_availability} in the appendix.
%\com{We had defined $T_L$ before. Is this different? GL yes. This is the amount of time spent transferring all they have to transfer (or that they can transfer) in a contact.}
%In what follows, we consider the system to be in steady state for the process of diffusion of the models.
%\subsection{Derivation of observation availability}

Let us now consider an observation which has been incorporated into the $m$-th model, and let $\tau$ denote its age. %Let $o(\tau,A)$ be the \textit{availability} of an observation with age $\tau$, i.e. the fraction of nodes with the $m$-th model whose local instance includes the given observation.
Let $o(\tau)$ denote the mean \textit{observation availability}, i.e., the mean fraction of nodes with the $m$-th model whose local instance includes the given observation at age $\tau$, at the mean field limit.
\begin{definition}
\textit{Given a node in the RZ, a given model, and a given observation of age $\tau$, at the mean field limit, the mean observation availability $o(\tau)$ is the probability to find that observation of age $\tau$ in the training set of an instance of that model at that node, assuming that 1) that node possesses an instance for that model, and 2) the given observation has been incorporated by training in at least one instance of that model in the whole RZ.}
\end{definition}
This definition implies that, given a node, a model and an observation of age $\tau$, the probability that it is included in an instance of that model on that node is $wao(\tau)$. 
%\com{why do we need model availability a???to distinguish merging from simple model transfer?}
%\com{Do we need both $o(\tau,A)$ and $o(\tau)$?}
%\com{Note: Theorem 1 is currently formulated for the multi-model case.}
Let $d_I$ denote the mean time taken by an observation to be incorporated through training, and $d_M$ the mean time required by an instance of the $m$-th model received by a node to be merged.
\begin{lemma}\label{lemma:arrivalrate} At the mean field limit, in the steady state for the model diffusion process, in the substable regime the mean arrival rate of merging tasks at a node $r$ is $r=MaSw^2g(1-b)^2$,
% \[
% r=MaSw^2(1-Y)g(1-b)^2
% \]
where %$\omega=\left\lceil \lambda (\tau_l-d_I) \right\rceil$, and 
$S$, $a$ and $b$ are given by \lref{lemma:model_availability}.
% \begin{itemize}
%     \item ;
% %    \item $Y=\sum_{\mathbf{e}\in\Omega} \prod_{k=1}^{\omega}o_k^{e_k}(1-o_k)^{1-e_k}\sum_{i=1}^{\omega}\prod_{j=1}^{i}o_j^{e_j}(1-o_j)^{1-e_j}\prod_{j=i+1}^{\omega}(1-o_j)$
%   \item  
% %\item $\tau^*$ is the observation age at which, at the mean field limit, observation availability is equal to one;
% %\item $\Omega$ is the set of all possible binary arrays $\mathbf{e}=(e_1,e_2,...e_j,...,e_\omega)$;
% %\item $o(\tau)$ is the mean observation availability at age $\tau$;
% %\item $o_j=o\left(d_I+\frac{j-1}{\lambda[1-(1-P)^\Lambda]}\right)$
% \end{itemize}
\end{lemma}
\forcr{For the proof, please refer to the extended version \cite{OntheLim18:online}.}
\forex{For the proof, please refer to \sref{sec:app_lemma:seedp} in the appendix.} %\com{proof to be simplified and updated accounting for $Y \approx 0$}.
\begin{lemma}\label{lemma:delays} Let $t^*$ be the mean sojourn time in the RZ. When the \textit{stability condition} 
\[\left(\frac{Mw\lambda \Lambda T_{T}}{N} +r T_M\right) \vee\left[\frac{1}{t^*2(1-r T_M)}\cdot\right.
\]
\be\label{eq:condition}
\left.\cdot\left(\frac{r T_M^2}{1-r T_M}+\frac{T_T(2-\frac{M\lambda \Lambda}{N} T_T)}{1-\frac{M\lambda \Lambda}{N} T_T}\right)\right]\leq 1
\ee
holds, $d_M$ and $d_I$ are given by 

\[    d_M= T_M + \frac{r T_M^2}{2(1-r T_M)}+\frac{Mw\lambda \Lambda}{N} {T_{T}}^2
    \]
    \be\label{eq:dm}
    d_I= \frac{1}{1-r T_M}\left(\frac{r T_M^2}{2(1-r T_M)}+T_T+\frac{\frac{Mw\lambda \Lambda}{N} T_T^2}{2(1-\frac{Mw\lambda \Lambda}{N} T_T)}\right)
    \ee
    %Finally, the seeding probability $P$ is given by $P=CCDF_{res}(d_I)$,  i.e. by the complementary cumulative distribution function of the residual sojourn time for a node in the RZ, evaluated in $d_I$. 
 \end{lemma}
 %\footnote{\com{maybe for computation of capacity limits, it would be better to not give priorities to merging or to training, and assume that tasks are scheduled at random among all those in queue.}}
\begin{proof} The system can be modeled as $M/D/1$ queue with two classes of customers, in which merging has non-preemptive priority over training.
For the high priority class, the arrival rate is $r T_M$, and it is $\frac{Mw\lambda \Lambda T_{T}}{N}$ for low priority.
From the condition of stability of the $M/D/1$ queue we derive the condition $r T_M+\frac{Mw\lambda \Lambda T_{T}}{N}\leq 1$. The derivation of delay formulas follows standard results from queuing theory for such systems. The second stability condition derives from imposing that the mean delays for both classes be less than the mean sojourn time in the RZ. 
%Indeed, when this latter condition is not satisfied, the likelihood of a (training or merging) task not being completed because of node exiting the RZ is non negligible, and thus the system is not able to satisfy the demand of computing tasks.
 \end{proof}
 \vspace{-0.1in}
  The following result gives an analytical expression for the incorporation rate, as a function of the  mean value of the observation availability at the mean field limit and over finite time intervals, for large RZ areas, hence for a large amount of nodes involved in the scheme.
\begin{theorem}\label{th:observation_availability} When the stability condition \eqref{eq:condition} holds, at the mean field limit the incorporation rate at age $\tau$ is given by $R(\tau)=\lambda o(\tau)$. The mean observation availability at age $\tau$ at the mean field limit $o(\tau)$ is given by the solution of the following delay differential equation:
\[
\dfrac{d o(\tau)}{d t}=\dfrac{b S(a)w^2}{T_S(a)}\left[(1-a)o(\tau)+\right.
%a o(\tau-d_M)(1-o(\tau-d_M))\right]+%\right .
\]
\be\label{eqtemp}
\left. +a o(\tau-d_M)(1-o(\tau-d_M))\right]-\dfrac{\alpha w}{N}o(\tau)
\ee
with initial condition 
\vspace{-0.2in}
\be\label{eq:bordercond}
    o(\tau)=\left\{
                \begin{array}{ll}
                  0 & \tau<d_I\\
                  \frac{1+(\Lambda-1)}{\left \lceil{a N}\right \rceil} &  d_I\leq \tau \leq d_I+d_M
                  %\frac{e^{a\frac{b S(a)}{T_S(a)}(\tau-d_I-d_M)}}{a N-1+e^{a\frac{b S(a)}{T_S(a)}(\tau-d_I-d_M)}}& \tau > d_I+d_M
                 \end{array}
              \right.
\ee
%$o(\tau)=0$ for $\tau<d_I$ and $o(\tau)=\frac{1+(\Lambda-1)P}{\left \lceil{a N}\right \rceil}$ for $d_I\leq \tau \leq d_I+d_M$.
\end{theorem}

For the proof, please refer to \sref{appendix:theorem_incorporationrate} in the appendix.\\%\com{(to be updated)}.\\
%\com{prove $o(\tau)$ always goes to 1 for growing $\tau$.}\\
This result states that the probability of observing a difference between any point of the trajectory of the availability of a given observation and the solution of \eqref{eqtemp} goes to zero as the RZ area (and hence the mean number of nodes in a RZ) grows. That is, in the mean field limit, the error made by considering a deterministic system characterized by $o(\tau)$, instead of the actual system goes to zero. Moreover, at any time $\tau$ from observation generation, for any $A$, $o(\tau)$ is the expected value of the observation availability at $\tau$ in a finite system. 

\begin{theorem}\label{th:staleness} $\forall i\in \mathbb{Z}^+$, let $\gamma_i=\sum_{k=1}^{i}\xi_k$, with $\xi_k$ i.i.d. $\sim Exp(\lambda)$. When the stability condition holds, the mean staleness $F$ of a model is lower bounded by 

\vspace{-3mm}

\be\label{eq_fresh}
\!F\!\geq\!\frac{\delta\sum_{i=1}^{\infty} i E\left[ o(\gamma_i)\middle
\vert\gamma_i\!\leq\!\tau_l\right]\prod_{j=1}^{i-1}(1\!-\!E\left[ o(\gamma_j)\middle
\vert\gamma_i\!\leq\!\tau_l\right])}{\sum_{i=1}^{\infty} E\left[ o(\gamma_i)\right]\prod_{j=1}^{i-1}(1\!-\!E\left[ o(\gamma_i)\middle\vert\gamma_i\leq \tau_l\right])}\!\!
\ee
\end{theorem}
For the proof, please refer to \forcr{the extended version \cite{OntheLim18:online}.}\forex{Appendix  \sref{appendix:th_staleness}.}\\

%% file: 5.model_capacity.tex
\vspace{-0.05in}
\section{Learning Capacity of a FG system}
In this section, we derive a set of results which allow evaluating the key performance indicators of our decentralized FG framework. 
%Namely the maximum of the mean amount of observations which can be learned, i.e. incorporated in a model using our scheme.
In what follows, we mostly refer to the case in which models are neural networks, but considerations may be extended to a large class of deep and shallow learning. 

We know from the literature (see \cite{friedland2017capacity} and references therein)
that when the total number of coefficients of an ML model is larger than the cardinality of the training dataset, then labels can be perfectly recovered. One of the interesting consequences of this result is that it establishes a linear relationship between model size $L$, i.e., the total amount of bits required to represent the model coefficients, and the maximum amount of information that the model can store.
Thus, if a model has length $L$, and $k$ coefficients are needed to encode one bit of information \emph{independently from other coefficients}, with a given ML algorithm, then the ratio $L/k$ is the largest amount of information (expressed in bits) that a model is able to store and recreate with vanishing little error \cite{friedland2017capacity}. 
%One of the interesting consequences of this result is that it establishes a linear relationship between model size $L$, i.e., the total amount of bits required to represent the model coefficients, and the maximum amount of information that the model can store. In what follows we denote with $k=L/Q$ the constant of proportionality between these two quantities. 
Note that typically data in a training set is not i.i.d., and the correlation among data points translates into looser requirements on the maximum amount of information that the model can store than the one indicated by these results.

% \subsection{Non-stationary data generation model}
One of the measures of effectiveness of FG training in a dynamic setting is the \textit{node stored information}.
%More precisely, we take the point of view of a user which, at a random time during his sojourn in the RZ, needs to use one or more of those models. We model such a request as a retrieval of all (i.e. a worst case) of the information stored in the model instances that the users possesses. 
%Such a quantity is a proxy of the quality of the services that, at that point in time, the user might get by using those models to implement them.
\begin{definition}
\textit{Consider a FG system at steady state for the process of model diffusion. The node stored information is the mean of the total amount of observations incorporated in all of the model instances possessed by a node,  with an age not greater than $\tau_l$.}
% the mean amount of information incorporated in all of the model instances possessed by \textit{a node} in our FG system, at steady state for the process of model diffusion. 
\end{definition}
%{\color{red}Is it clear, in Definition 8, that the observations are not incorporated in the node but in the MODEL instance of the node? }
\begin{lemma}\label{lemma:amountofstoredinfoatanode} When the stability condition \eqref{eq:condition} holds, the node stored information
is $M wa \min\left(\frac{L}{k},\lambda \int_0^{\tau_l} o(\tau)d\tau\right)$.
% \vspace{-0.2in}
% \be\label{eq:storedinfo_nonstationary}
% M wa \min\left(\frac{L}{k},\lambda \int_0^{\tau_l} o(\tau)d\tau\right)
% \ee
\end{lemma}
\begin{proof}The probability for a user to possess an instance of a given model is $wa$, and thus $Mwa$ is the mean number of model instances it possesses, at the mean field limit. For each model, the mean amount of observations incorporated is the minimum between the maximum amount that the model can store, and the mean of the total number of observations incorporated in an instance of that model with an age between $0$ and $\tau_l$, given by the integral of the observation incorporation rate within $[0,\tau_l]$. 
%\com{(vm: OK, but it is not clear to me why we have to further multiply the seeding probability and the rate of observation generation per user, i.e., the rate of seeded observation per user. I mean, $L/K$ is already the (max) information that a model can store, so $Ma\min(\cdot)$ looks already as the per-node capacity we are looking for, as stated in the paragraph before the lemma.)}
\end{proof}

Note that at the end of  Section~\ref{ss:assumptions} we have assumed that, in a model which is already storing an amount of information $L/k$, the incorporation of the new data point implies discarding the oldest data point in the model. This assumption corresponds to an ideal case, as in practical cases many models simply degrade their accuracy when incorporating more points than they can store. Thus, \lref{lemma:amountofstoredinfoatanode} gives an upper bound on the total amount of information stored at a node.  

% Another parameter of relevance is the total amount of information stored in the GL system, defined as the set of all the observations stored in all instances present in the RZ, and whose age is not larger than $\tau_l$. 
% \begin{corollary}When the stability condition holds, for a maximum observation age $\tau_l$, the mean total amount of information stored in the GL system is given by 
% \be\label{coroll}
% MwN\min\left(\frac{L}{K},(\tau_l-d_I) \lambda [1-(1-P)^\Lambda] \right)
% \ee
% \end{corollary}

Another key performance indicator of our FG scheme is the \textit{learning capacity}.

\begin{definition}
\textit{In a FG system at steady state, the learning capacity is the maximum achievable ratio between the node stored information and the total observation arrival rate.} %\textcolor{red}{\bf(This definition is not coherent with what defined  half page before, in which $Q$ is a learning capacity expressed in bits  --  WHICH ONE DO WE USE???)}
\end{definition}

The learning capacity gives an idea of up to which point our FG scheme is able to effectively  incorporate in the models possessed by nodes in the region those observations which are continuously generated within it, and thus to keep the models up to date. Specifically, the learning capacity of an FG scheme is the solution of the following problem:

\vspace{-0.16in}

\begin{myproblem}[\textbf{FG learning capacity}\vspace{2mm}]\label{problem:opt:nonstationary}
\begin{align}
\vspace{1in}
%&\nonumber\\
%
& 
\underset{M,L}{\text{maximize}}\;\; wa \min\left(\frac{L}{\lambda 
 k}, \int_0^{\tau_l} o(\tau)d\tau\right) &\nonumber\\
 &\;\;\text{subject to}: \quad \quad&\nonumber\\
%& \frac{M\lambda \Lambda T_{T}}{N} +r T_M\leq 1\;& \label{constraint:space}\\
%& \frac{1}{1-r T_M}\left(\frac{r T_M^2}{2(1-r T_M)}+T_T+\frac{\frac{M\lambda \Lambda}{N} T_T^2}{2(1-\frac{M\lambda \Lambda}{N} T_T)}\right)\leq t^* \;& \label{constraint:space2}\\
 %& \quad \quad P=CCDF_{res}(d_I)\quad&\nonumber\\
 &\quad \quad\text{Equations}\;(\ref{eq:steady_aa}),\;(\ref{eq:condition}),\;(\ref{eqtemp})&\nonumber\\
 %& a=1-\frac{2\alpha T_S(a)}{DRS(a)b}&\label{constraint:am1}\\
 %& a > 0 & \label{constraint:fl1}\\
% & \frac{\lambda \Lambda T_{T}}{N}\leq 1 &\\
 &  \quad\quad M\geq  1,\quad L\geq L_m\quad&\nonumber
  %& \bar{o}(\tau)>0& 
 \end{align}
%\end{equation}
\end{myproblem}
 % A similar optimization problem can be defined for the maximum amount of information stored in the GL system, with the same constraints as Problem 1, and with \eref{coroll} as utility function.\com{I would not even mention this parameter. Defining it entails defining a way to retrieve that operation ("a read operation"), and that would bring us far from the core of the work.}\\
 % \com{ISSUE 2. NLC formula is not ok, the first term should be in n of observations, as the second. Note that we are implicitly assuming  all observations have a same "size". }
\begin{prop}\label{prop:contentsize} If $(M^*, L^*)$ is a solution of Problem 1, then $L^*=L_m$.
\end{prop}
\begin{proof} %\textcolor{cyan}{\bf(Revise it? I cannot follow the arguments)} 
$\frac{wL a}{\lambda k}$ is directly proportional to the \textit{area capacity} of the floating system represented by the RZ area and the $M$ models floating within it (see \cite{kaisar2022decentralized} and references therein).
Thus, for the first element of the minimum in the objective function, Proposition 1 follows from its properties. As for the second element, we note that $w$ does not depend on L. The integral $\int_0^{\tau_l} o(\tau)d\tau$ monotonically decreases when $a$ decreases. At the same time, when keeping $M$ constant, $a$ decreases monotonically as $L$ increases. Indeed, holding constant everything else, the lower the content size, the lower the amount of contact time spent in content transfers which do not complete due to the end of contact time. 
%Moreover, $P$ increases when decreasing $L$, as this brings to a lower computing load, and thus a higher probability to have enough time to complete a training task. 
Finally, for the same reason, the left member of constraint \eqref{eq:condition} decreases when decreasing $L$, increasing the maximum value of $M$ which satisfies all constraints in Problem 1.
%Note that in real systems, in which communication overhead is non negligible, the optimal model size is not the minimal one ($L=1$). For those systems however, the capacity computed in the ideal case represents an upper bound to what achievable when overheads are accounted for.
\end{proof}
 Therefore, Problem 1 can be cast into a maximization problem over $M$ only, with model size equal to its minimum value. Thus, it can be solved efficiently with greedy approaches.

%% file: 5.num_evaluation.tex
\section{Numerical evaluation}
\label{sec:performance_evaluation}

\begin{figure*}[!t]
	\centering
	\begin{subfigure}[b]{\columnwidth}
		\centering
\includegraphics[width=0.9\columnwidth, height=2.1in]{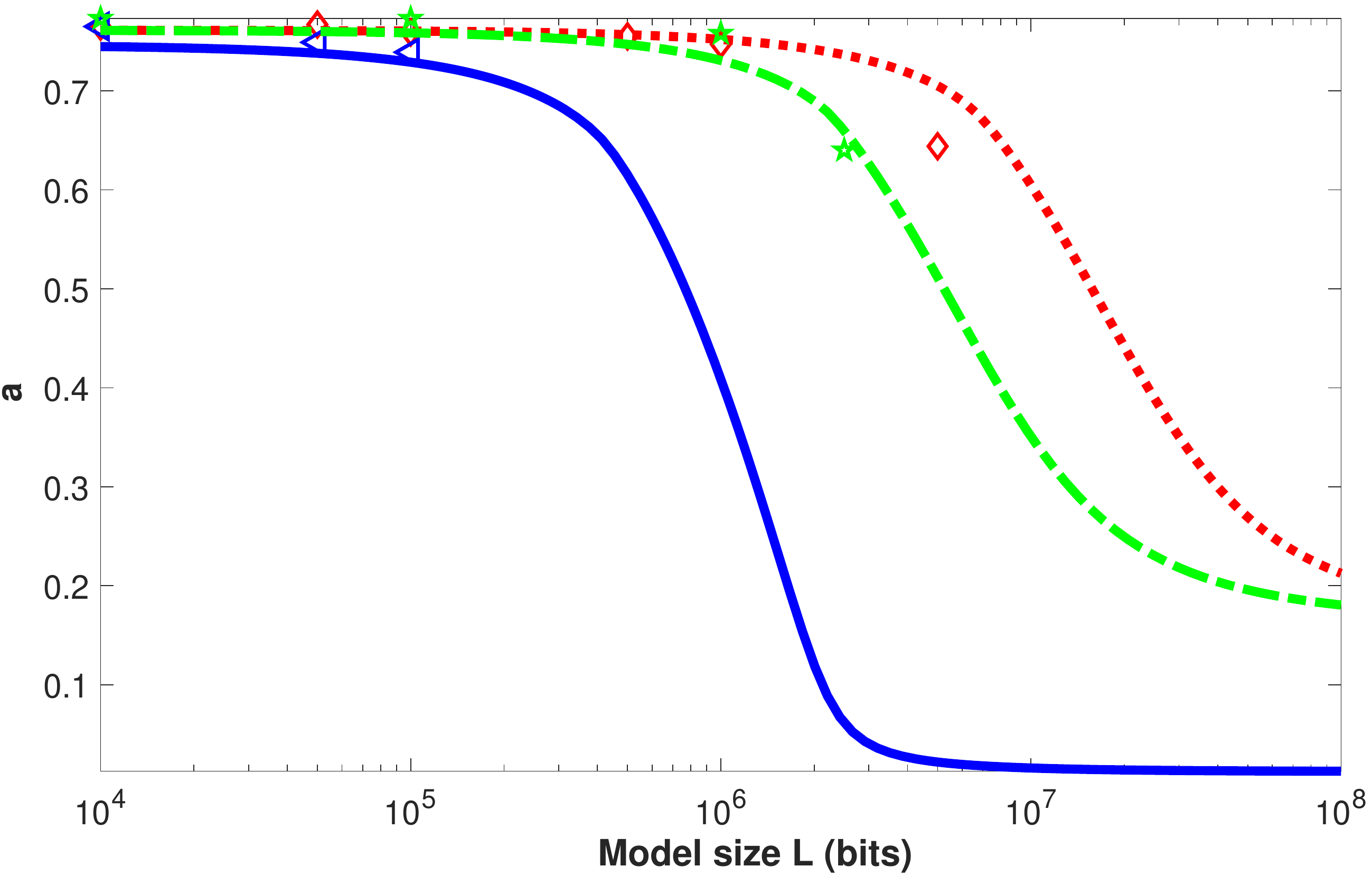} 
\vspace{-1mm}
\caption{Mean model availability.}
		\label{fig:m_availability}
		\vspace{-1mm}
	\end{subfigure}%
 %\hspace{0.1pt}
	\begin{subfigure}[b]{\columnwidth}
		\centering
		\includegraphics[width=0.9\columnwidth, height=2.1in]{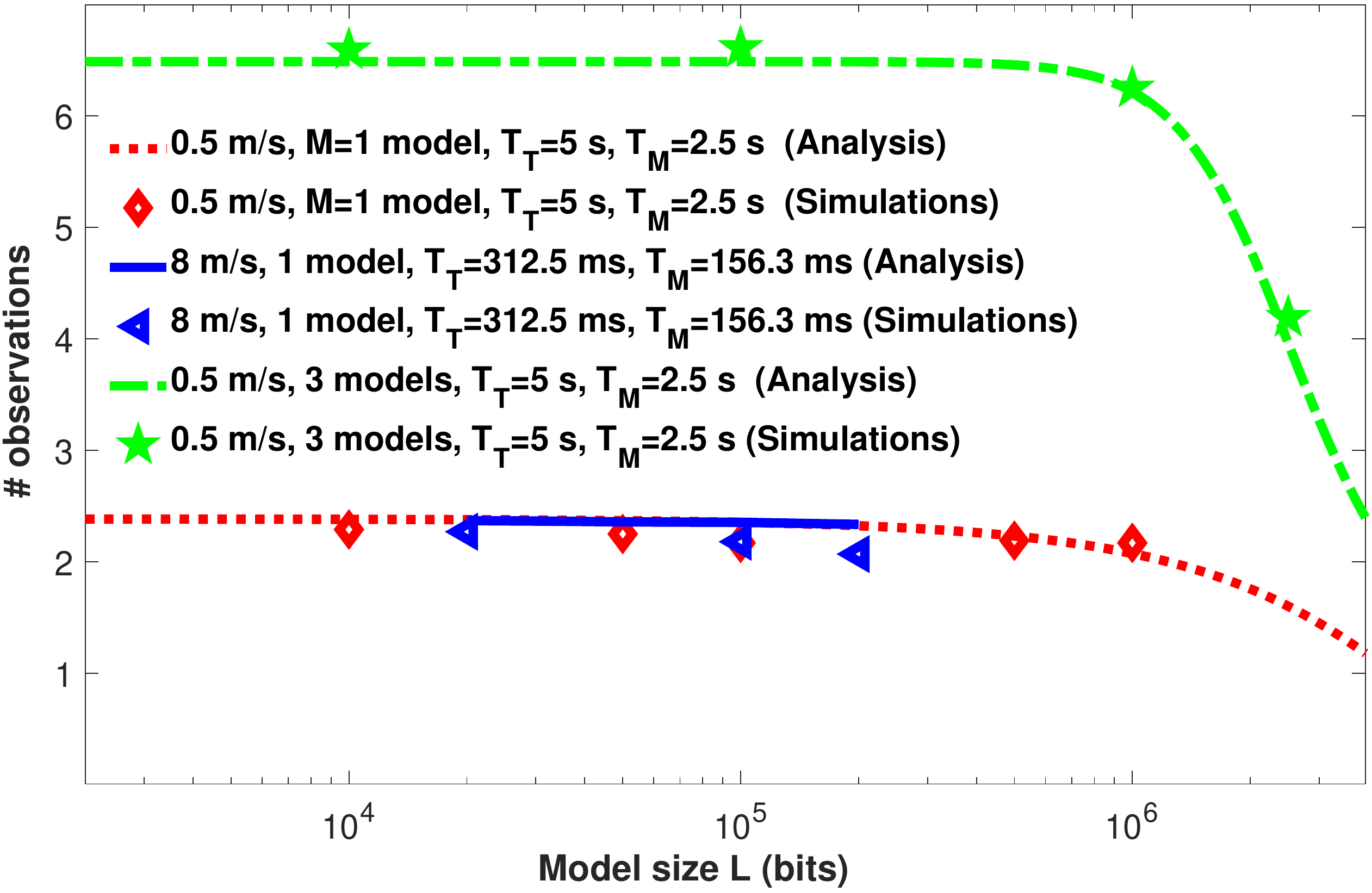} 
\vspace{-1mm}
\caption{Node stored information.}
\label{fig:m_availability}
\vspace{-1mm}
	\end{subfigure}
 \caption{Mean model availability $a$ (left) and node stored information (right) in a FG system versus model size $L$, for different values of training time ($T_T$) and merging time ($T_M$), of node speed, and of number of models $M$. Simulation results (represented by markers) are averages with a $95\%$ confidence intervals of at most $5\%$ (not shown in the figure).}
\label{fig:twographs}
        \vspace{-5mm}
\end{figure*}

In this section, we use the mean field model presented in this paper to characterize the limit performance (in the mean field sense)
of a FG system as a function of the main system parameters, and we assess the accuracy of the model by means of detailed simulations.

In the simulator, we consider a square area with side $200$ m, at the center of which is a circular RZ, whose radius is $100$ m.
A fixed number of nodes (200 is the default value) move according to the Random Direction Mobility Model, with reflections at the boundary of the simulated area. Nodes have
a transmission radius of $5$ m, and enough storage to accommodate an arbitrary number of models. 
When two nodes are in contact, the channel data rate is taken to be constant, equal to $10$ Mb/s.
The default duration of training and merging tasks is $5$ s and $2.5$ s, respectively. The maximum age for an observation is $5$ min, and the default model size is $10$ kb (hence, model exchanges require 2 ms). 
At the beginning of each simulation run, nodes are randomly positioned within the simulation area, and they only possess a default model instance. 
%As the simulation progresses and observations are included by nodes in the RZ, model instances are generated and distributed in the RZ via opportunistic replication\com{all this shouls already have been explained in the system model...}.
%The simulation time advances in slots of constant size, equal to the time needed to send a model instance while receiving another one, hence using half of the available D2D bandwidth {\color{red}Does the previous statement hold for a model size of 10kb?}. The slot duration is chosen with the objective of minimizing the effects of time quantization. 
%and in particular on the errors in detecting when two nodes are in range or when a node is within a given RZ.
Simulation time is divided in equally sized slots, whose duration has been chosen in such a way as to minimize quantization effects on the accuracy of results. Simulations are run until steady state is reached, and transients are discarded.\\
% \begin{figure}[t!]
% \centering
% \includegraphics[width=\columnwidth, height=2.1in]{fig/model_availability_v4.pdf} 
% \caption{Mean model availability in a \com{FG} system versus model size, for different values of training time (TT) and merging time (MT), of node speed, and of number of models. Simulation are with a $95\%$ confidence interval of at most $5\%$.}
% \label{fig:m_availability}
% \end{figure}
% \begin{figure}[t!]
% \centering
% \includegraphics[width=\columnwidth, height=2.1in]{fig/NSI_v1.pdf} 
% \caption{Node stored Informations in a \com{FG} system versus model size, for different values of training time (TT) and merging time (MT), of node speed, and of number of models. Simulation are with a $95\%$ confidence interval of at most $5\%$.}
% \label{fig:m_availability}
% \end{figure}
In a first set of experiments, we have assessed model accuracy by comparing analytical and simulation results for the mean model availability and the mean number of observations stored by each node, as a function of model size. As \fref{fig:twographs} shows, the estimates derived with our mean field approach well match the simulation results across different values of node speed, of model size, of number of models, and of training and merging times. We can also see that when model size gets close to the maximum amount which can be sustained by the system (due to limitations in contact duration among nodes and in channel throughput), our mean field model yields slightly optimistic results, due to the difference between finite systems and their mean field limit.\\ 
%\com{This gives an idea of the range of validity of our approach.}
\begin{figure}[t!]
\centering
\includegraphics[width=\columnwidth, height=2in]{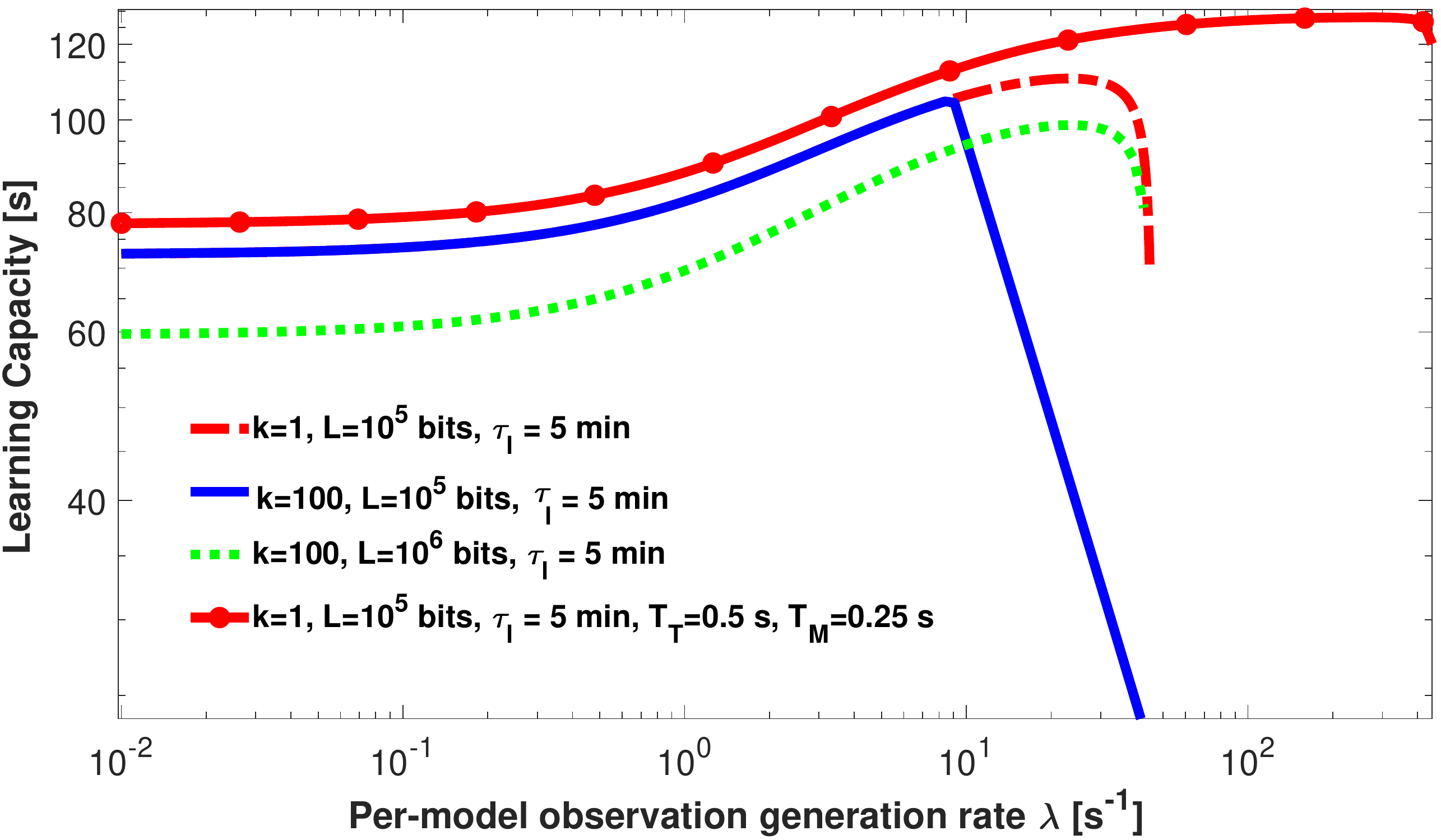} 
\caption{Learning capacity in a FG system versus per-model observation generation rate $\lambda$, as a function of observation lifetime $\tau_l$, model size and capacity ($L/k$), and training/merging times $T_T$ and $T_M$.
%{\color{red} Is the legend for blue line correct? is it M or L = 10^5?}
%\com{Maybe with a linear vertical scale the curves are less flat?}
}
\label{fig:learning efficiency}
\vspace{-5mm}
\end{figure}
To investigate the properties of the FG learning capacity, in \fref{fig:learning efficiency} we plot the optimal values deriving from the solution of Problem 1 as a function of the per-model observation generation rate.
%Recall: node stored information and learning capacity is the min between two terms, one accounting for model capacity, and the other for the mean amount of observations which can be incorporated in a model through the FG training process.
% \footnote{\com{Parameter is of some importance (it shows how learning capacity evolves with load) but limited by the fact that one does not really choose M in order to optimize performance, but rather revenues (it increases M until either the system cannot support them, or it becomes economically suboptimal)}}  
As the plot suggests, when the model capacity $L/k$ is large enough, the learning capacity grows with the observation arrival rate (and thus with the amount of information which needs to be captured). This is due to the fact that a higher rate of arrivals of observations leads to a larger model availability, and thus to a higher chance for observations to be incorporated in a model and to diffuse in the RZ. 
%Thus when the system is in its linear regime (i.e. when performances are not limited by finiteness of contact duration/throughput of wireless channel, AND by maximum computing load sustainable by a node)
%, as the load  increases the system efficiency improves. 
The experiments also show that learning capacity peaks at a point at which further increases in load are not sustainable due to node computing capacity or data transfer capacity limits. For higher loads, it decreases sharply, until the system is no longer stable, i.e., it is no longer able to sustain the demand for computing (due to training and merging) that the arrival rate of observations induces. 
The curves also show that decreasing training and merging times by a factor ten with respect to their default values increases the maximum value of observation arrival rate beyond which the system is unstable by a factor ten. Not surprisingly, decreasing those computing times increases also learning capacity, since it accelerates the diffusion of models and of observations within the considered region. When model capacity is not ``large enough", even if the system is in its linear regime, the amount of observations incorporated (and thus the learning capacity) grows until it hits the model capacity limits. After that, it remains constant, implying a learning capacity that decreases with the inverse of the load.\\
% A higher maximum observation lifetime implies a higher Learning capacity. Curiously, by comparing the red and green curves, a larger model size implies a lower Learning capacity, despite the maximum capacity of the model is larger. this is due to the fact that larger models take more time to be exchanged \com{(and trained, but we did not model that)}, and thus, keeping all else constant, they generally have a lower availability, thus negatively impacting the mean amount of observations incorporated at a node.
\begin{figure}[t!]
\centering
\includegraphics[width=0.9\columnwidth, height=2.1in]{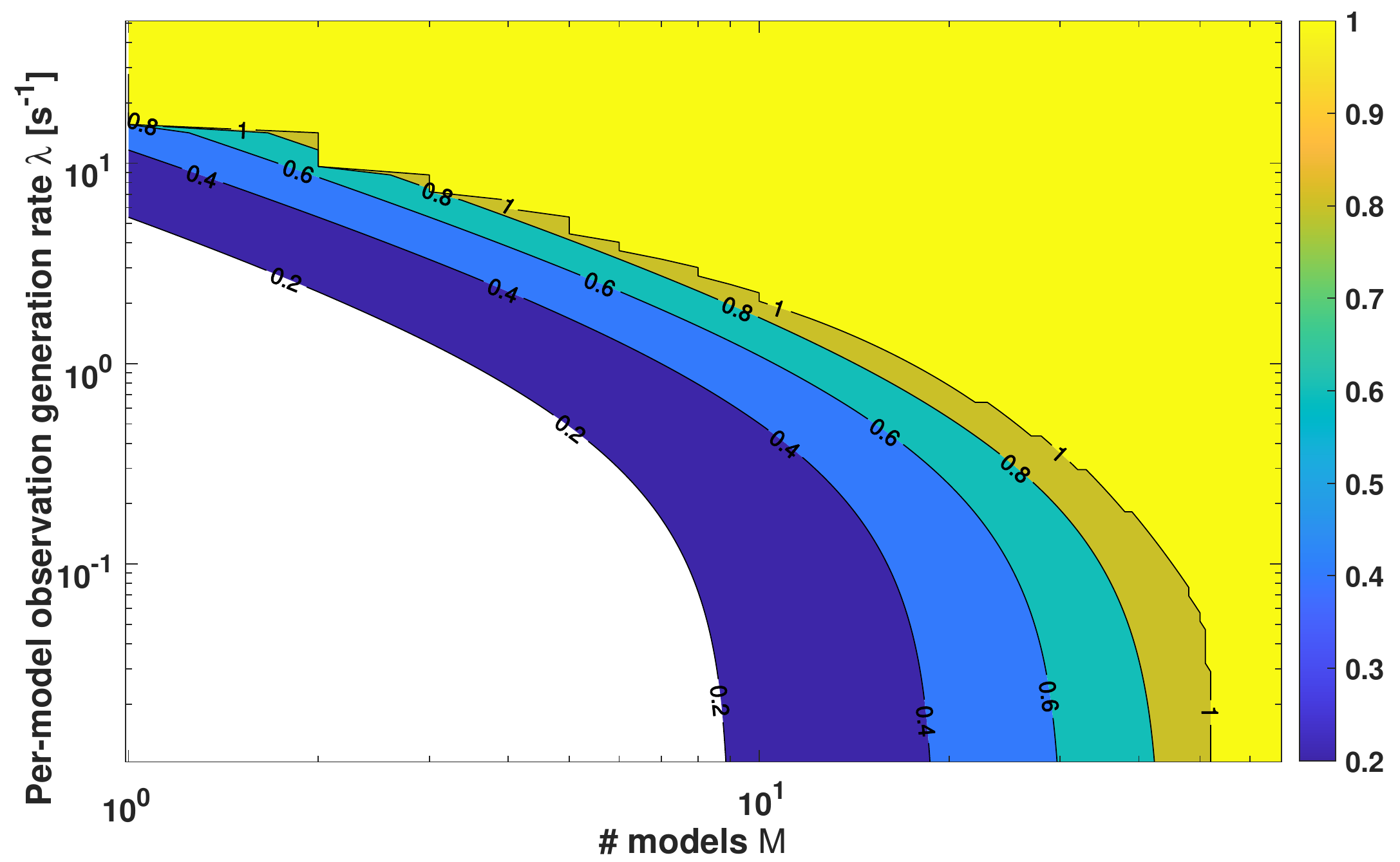}\caption{Value of the left hand side of stability condition \eqref{eq:condition} versus number of models $M$ and per-model observation generation rate $\lambda$, for observation lifetime $\tau_l=300 s$, model size $L=10$~kb, and $k=1$.}
\label{fig:stability}
\vspace{-5mm}
\end{figure}
Next, with \fref{fig:stability} we characterize the stability region of the system, i.e., the range of values of load and of number of models for which \eqref{eq:condition} holds, so that the FG system is actually able to incorporate new  observations in a timely manner. In the figure, the color scale corresponds to the value of the left-hand-side of~\eqref{eq:condition}, which must be below one. Hence, the 
yellow area above the isotonic line at value 1 represents unstable points. The same line also illustrates the tradeoff between number of models and observation generation rate. In this specific case, the system is stable with at most about 40 models with observations generated every 100 seconds, but also with just one model whose observations can be generated as frequently as about twenty times per second, which is a rate suitable to capture a video.  
From \fref{fig:stability}, it emerges that for small values of per-model observation generation rate (with respect to observation lifetime), the system is not limited by node computing load due to training and merging, rather by the maximum amount of models which it can sustain 
%(i.e,. allow to persist with very high probability in the region) 
for a given model size. Indeed, from the figure, stability depends mainly on number of models in the system.
Note that, for a given maximum age of observations, there is a lower bound on the minimum per-model observation generation rate, below which models are likely to incorporate very few (if any) observations, thus making the whole FG system useless.
When observation generation rate increases, the system performance is progressively more dependent on the maximum computing load it can tolerate. Clearly, in this ``computing-limited" regime, an increase in the number of models entails not only an increase in the training load, but also an increase in overall observation arrival rate, and thus merging load. Indeed, as the plot shows in this regime, keeping the stability condition unaltered after such an increase would entail decreasing the per-model observation arrival rate, with a relationship of (roughly, logarithmically) inverse proportionality. 
Note that discontinuities in the plot are due to quantization in the number of models.

%\subsection{Staleness}
 \begin{figure}[t!]
\centering
\includegraphics[width=\columnwidth, height=1.8in]{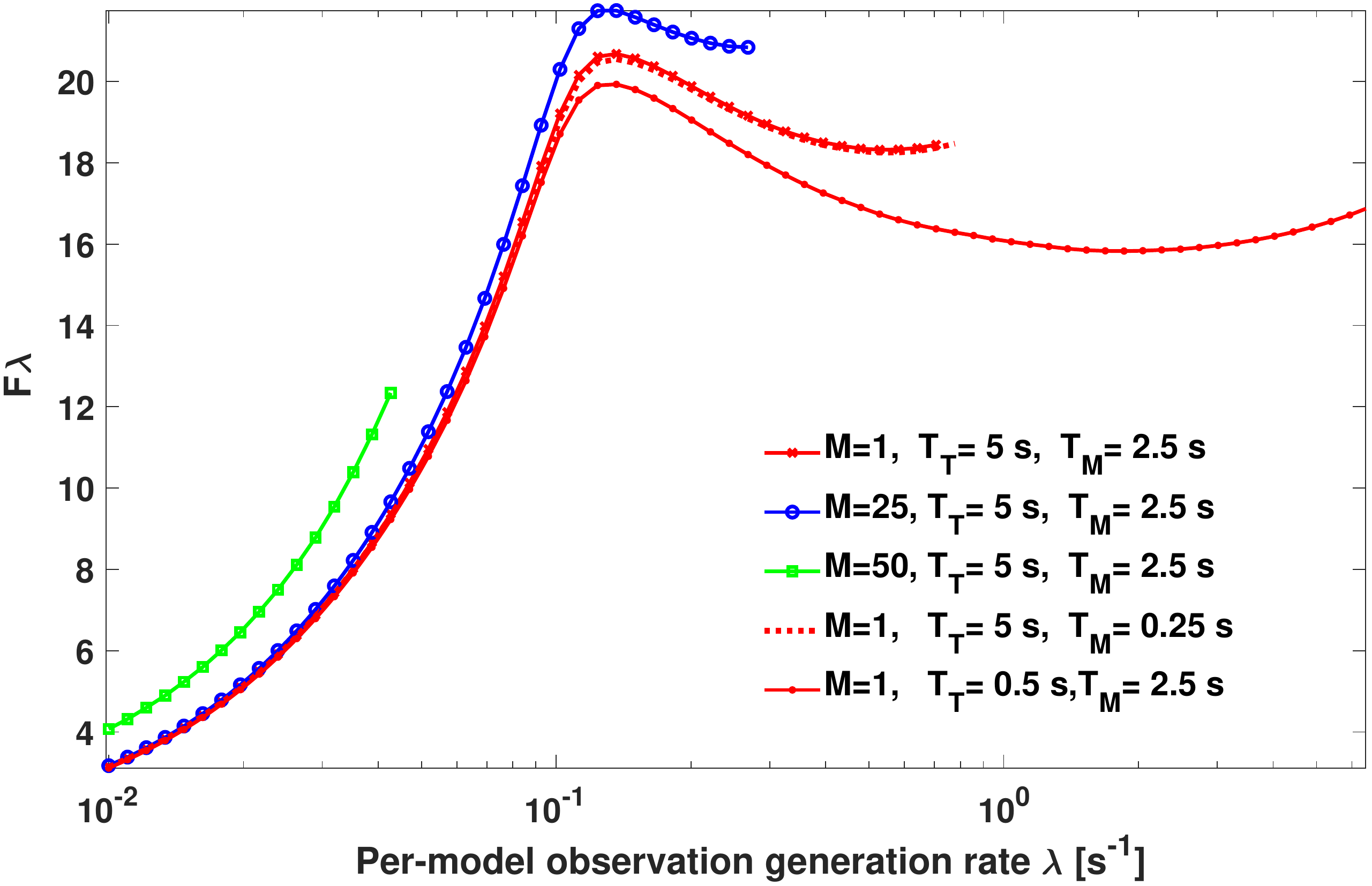} 
\caption{Ratio of model staleness $F$, computed via its bound~\eqref{eq_fresh}, over per-model mean observation interarrival time $1/\lambda$, versus per-model observation generation rate $\lambda$.  
Observation lifetime is $\tau_l=300$ s, $k=1$, with model size $L=10$~kb.}
\label{fig:staleness}
\vspace{-5mm}
\end{figure}

Coming to model staleness, we report in Fig. \ref{fig:staleness} the model prediction normalized to the per-model mean observation interarrival time, versus the per-model observation generation rate $\lambda$, in the case of observation lifetime $\tau_l=300$ s, $k=1$.
As we can expect, staleness decreases for increasing observation generation rate, as it is the 
%absolute  measure 
age of the most recent observation incorporated in a model.
However, we are interested in how quickly an observation is incorporated into a model. 
%For this reason in \fref{fig:staleness} we have plotted the ratio of model staleness over per-model mean observation interarrival time, versus per-model observation generation rate. 
As we can see, for low values of $\lambda$, an increase of $\lambda$ brings the FG system to incorporate more quickly an observation, as a higher observation arrival rate implies a larger model availability and thus a more effective diffusion of observations in the region. However, for large values of $\lambda$, the average computing load of nodes becomes high enough to actually slow down the process of diffusion of observations. Thus, in this operating regime the normalized model staleness first peaks and then decreases. For high enough values of $\lambda$ (where curves stop), the system becomes unstable. As expected, this takes place for lower values of $\lambda$ in the case of a higher number of models. 
The figure also shows that by increasing the number of models floating in the RZ, the normalized staleness increases, as expected. However, the increase is less than linear with the number of models. At the peak, moving from $M=1$ to $M=25$ models only incurs a staleness increase of about 10\%. From this observation we conclude that FG scales relatively well with $M$.
%    \item Finally, we have assessed the impact of the computing cost of the training and merging operations, on the above mentioned ratio. (...) 

%% file: 6.conclusions.tex
\vspace{-0.1in}
\section{Conclusions}
\label{sec:conclusions}
In this paper we have developed a mean field model to investigate the limit performance of Floating Gossip, a distributed learning scheme based on Floating Content for the opportunistic exchange of ML models stored by mobile users. %We validate our mean field model by comparison with a detailed simulator, and we study the maximum amount of data that users can incorporate into their models, as a function of the main system parameters.
Our study shows that Floating Gossip can be very effective in implementing opportunistic, fully distributed, privacy preserving, training and update of ML models in a cooperative manner.
%, until the communication and computing resources in the system reach overload. 
Differently from centralized training schemes, whose limit performance is heavily degraded by the effects of node churn and mobility, gossip-based  schemes have proven to be very robust against those factors, and to depend heavily on nodes' ability to harvest data about the environment in which they operate. 
%{\color{red}{POSSIAMO METTERE QUALCOSA DI PIU' SPECIFICO COME SUMMARY OF RESULTS OVVERO COME TAKEAWAY MESSAGE?}}
\section*{Acknowledgments}
This work was partially supported by COST INTERACT, by SNF Dymonet project, and by the AEON-CPS project (TSI-063000-2021-38) funded by the Spanish Ministry of Economic Affairs and Digital Transformation.

%% file: 7.appendix.tex
%\appendix
\begin{appendix}
\subsection{Proof of \lref{lemma:model_availability} (outline)}
\label{sec:app_lemma:model_availability}
First, it is easy to see that, with the given assumptions, the number of non-default model instances possessed by a node is distributed as a binomial $B(Mw,a_m(t,A))$. The mean number of non-default model instances possessed by a node in stationary state is thus $Ma_m(t,A)w$.\\
To derive the mean value of the variables $a_m(t,A)$ and $b(t,A)$ at the mean field limit, we first derive the differential equations for the drift \cite{kolesnichenko2012applying,bortolussi2013continuous}.
Let $N_a$ denote the mean number of nodes in the region possessing a given model at time $t$.
We have therefore $N_a=a_m(t,A) A D$, where $D$ is node density.
The rate of change of $N_a$ is given by the sum of three components. The first is given by those nodes which complete a model instance transfer in a contact with another node, $\frac{A D b(t,A)}{T_s}a_m(t,A)(1-a_m(t,A))S(t)$. The second term is given by the rate at which a node receives an observation for that model (i.e., $\frac{\lambda\Lambda}{ D A}$) multiplied by the probability that the given node does not have the model yet (i.e., $(1-a_m(t,A))w$). Finally, the term $\frac{\alpha}{N}w a(t)$ accounts for the rate at which nodes with a model instance leave the given region.\\
Let us now consider the rate at which the mean number of busy nodes in the region evolves over time. Its increase is due to contacts (taking place at a rate of $g A$), among two non-busy nodes (with a probability $(1-b(t,A))^2$). Its decrease is due to the end of the process of model exchange with another node. Such an event takes place with a mean rate
$\frac{A D b(t,A)}{T_s}$. Putting all together, and normalizing by $A D$, we get
\begin{align*}
\begin{cases}
\displaystyle 
\dfrac{d a_m(t,A)}{d t}=\dfrac{b(t,A)}{T_s(t)}a_m(t,A)(1-a_m(t,A))S(t)w^2+
\vspace{3mm}
\\
+\dfrac{\lambda\Lambda(1-a_m(t,A))w}{ D A}-\dfrac{\alpha}{DA}w a_m(t,A)
\vspace{3mm}
\\
\dfrac{d b(t,A)}{d t}=2g(1-b(t,A))^2 - \dfrac{b(t,A)}{T_S(t)} -\dfrac{2\alpha}{DA}b(t,A) 
\end{cases}
\label{eq:ode}
\end{align*}
% Thus the system is able to reach \textit{full stationarity}, not just quasi-stationarity.\\
From this drift expression, by applying standard results, it is easy to prove the convergence to the mean field limit, from which the steady state expressions for $a$ and $b$ are obtained. 

\forex{\subsection{Proof of \lref{lemma:arrivalrate} (outline)}
\label{sec:app_lemma:seedp} 
If $g$ is the mean contact rate, of all contacts, those which may bring to an exchange of models are those in which both nodes are not busy. The probability of such an event is $(1-b)^2$. The mean amount of models received during an exchange is derived in a similar way as $\gamma$ in Lemma 1.}

\subsection{Derivation of \thref{th:observation_availability}}
\label{appendix:theorem_incorporationrate}
%$\lambda[1-(1-P)^\Lambda]$ is the rate of arrivals of observations which are incorporated in at least one instance of the model to which they are associated.
%Then \eref{eq_incorprate} descends from the definition of incorporation rate.\\
%As for the derivation of $o(\tau)$, 
We derive the drift, which describes the average local variation of observation availability over time. We assume the system satisfies the  \textit{homogeneous conditions}. That is, that at $t=0$ the mean number of nodes per unit surface possessing a given model is the same for all models; that at any time instant, nodes possessing a given model are uniformly distributed within the region; and that the probability of a node to have that model is independent from the probability of any other node to have the same model. 
\begin{lemma}\label{lemma:balanceequations} When the system satisfies the homogeneous conditions, the \emph{drift} of the CTMC is given by %the following delay differential equation (DDE)
$\dfrac{d o(\tau,A)}{d \tau}=\dfrac{b S(a)}{T_S(a)}a o(\tau-d_M,A)(1-o(\tau-d_M,A))$,
with initial condition $o(\tau,A)=0$ for $\tau< d_I$, and $o(d_I,A)=\frac{1}{\left\lceil N a\right\rceil}$.
\end{lemma}
\begin{proof} The derivation goes along the same lines as that of Lemma 1. 
 Let $N_o(\tau,A)$ denote the number of nodes in the RZ possessing a model instance including the given observation, at a given time $\tau$ from observation generation, for a RZ area A.
The rate at which this quantity varies over time is given by the sum of three components. The first is due to contacts between, on one side, a node with a model which incorporates the given observation, and a node possessing only the default model.
The mean rate at which nodes exit the busy state is given by the ratio between the mean number of busy nodes at time $\tau$ in the region, given by $\bar{N}(A)b$, and the mean time taken by an exchange between two nodes, $T_S(a)$. Thus the rate of relevant events is $\frac{\bar{N}(A)b}{2T_S(a)}$. Moreover, let us consider one of these terminating exchanges. The probability that the model was transferred to the node possessing only the default model during such exchange is equal to the probability that the sender node had the model and that the receiving node did not have it, given by $2a(1-a)$, multiplied by the probability that the transferred model incorporated the given observation, $\frac{N_o(\tau,A)}{a\bar{N}(A)}=o(\tau,A)$, and by the probability that the model was transferred successfully, given by $S(a)$. Summing up, the first contribution is $\frac{\bar{N}(A)b S(a)}{T_S(a)}a(1-a)o(\tau,A)$.
Similarly, for the second contribution, which accounts for the case in which the exchange is among two nodes each with a non-default model, we have $
\frac{\bar{N}(A)bS(a)}{T_S(a)}a^2o(\tau-d_M,A)(1-o(\tau-d_M,A))$. This correspond to the rate at which nodes conclude the merging operation, after receiving the instance with the given observation. Thus the observation availability that matters is the one at the moment in which the exchange took place, i.e., at time $\tau-d_M$. The third contribution is given by nodes exiting the RZ, and it is given by the rate at which nodes exit the RZ (i.e., $\alpha$) multiplied by the fraction of those exiting nodes which has the model and incorporates the observation (i.e., $a o(\tau,A)$). Putting all together, and normalizing, we get 
the final expression of the drift.\end{proof}
% we have
% \[
% \dfrac{d N_o(\tau,A)}{d t}=\dfrac{\bar{N}(A)b S(a)}{T_S(a)}\left[a(1-a)o(\tau,A)\right .
% \]
% \be
% \left . +a^2 o(\tau-d_M,A)(1-o(\tau-d_M,A))\right]-\alpha a o(\tau,A)
% \ee
% Now, we have that $N_o(\tau,A)=\bar{N}(A)a o(\tau,A)$. We normalize by the mean total number of nodes with the model within the RZ, i.e. by $\bar{N}(A)a$. We thus get
% \[
% \dfrac{d o(\tau,A)}{d t}=\dfrac{b S(a)}{T_S(a)}\left[(1-a)o(\tau,A)\right .
% \]
% \be\label{eqtemp}
% \left . +a o(\tau-d_M,A)(1-o(\tau-d_M,A))\right]-\dfrac{\alpha}{\bar{N}(A)}o(\tau,A)
% \ee
% From \eref{eq:steady_aa}, we have that
% \[
% \frac{2\alpha}{DR}=(1-a)\frac{b S(a)}{T_S(a)}
% \]
% By substituting this into \eref{eqtemp} we get \eref{eq:dde}.

% \be\label{eq:solutiondde_2}
%     \bar{o}(\tau)=\left\{
%                 \begin{array}{ll}
%                   0 & \tau<d_I\\\\
%                   \frac{1}{\left \lceil{a N}\right \rceil} &  d_i\leq \tau \leq d_I+d_M\\\\
%                   \frac{e^{a\frac{b S(a)}{T_S(a)}(\tau-d_I-d_M)}}{\left \lceil{a N}\right \rceil-1+e^{a\frac{b S(a)}{T_S(a)}(\tau-d_I-d_M)}}& \tau > d_I+d_M
%                  \end{array}
%               \right.
% \ee
%\end{proof}
%$o(\tau,A)=0$ for $\tau< d_I$, and $o(d_I,A)=\frac{1}{N a}$
\vspace{-0.1in}
The rest of Theorem 1 proof follows the  standard procedure for mean field convergence, i.e., convergence of initial conditions, and condition on the drift of the Population Continuous Time Markov Chain (PCTMC) associated to the system.
As its drift is continuous for $\tau\geq d_I$, let $\bar{o}(\tau)=\lim_{R\rightarrow\infty}o(\tau,A)$. As our sequence of PCTMC models satisfies these properties, by Theorem 1 in \cite{kolesnichenko2012applying}  for any finite time horizon $T\leq \infty$, the sequence of population models associated to $A$ converges \textit{almost surely} to the dynamics of the ordinary differential equation in~\eqref{eqtemp}, which proves Theorem 1. 
\forex{\subsection{Derivation of \thref{th:staleness}}
\label{appendix:th_staleness}
% \begin{proof} The mean rate at which observations associated to a given model arrive in the RZ is $\frac{\lambda\Lambda}{M}$. The mean rate at which they are included successfully in at least one model instance is thus $[1-(1-P)^\Lambda]\frac{\lambda\Lambda}{M}$.

% \end{proof}

\begin{proof} (sketch). For a given time $t$, let $i\in\mathbb{Z}^+$ be the label of the $i$-th last observation generated in the system, counting backward in time from $t$. We aim at computing the expectation of the time $\tau$ since the generation of the last observation incorporated in a given model. We have
$E[\tau]= \frac{\sum_{i=1}^{\infty}E[\tau|i]P(i)}{\sum_{i=1}^{\infty}P(i)}$
where $E[\tau|i]$ is the expectation of $\tau$ conditioned on having observation $1$ to $i-1$ not included in the model, and observation $i$ included, while $P(i)$ is the probability of such an event. As arrivals are Poisson with intensity $\lambda$, $E[\tau|i]=\frac{i}{\lambda}$. $P(i)$ is computed taking into account that arrivals are independent, and that the probability of an observation generated $x$ seconds ago to be included in a given model is $o(x)$. We have thus
$F=\frac{\delta\sum_{i=1}^{\infty} iE_{\xi_1,...,\xi_i}\left[ o(\gamma_i)\prod_{j=1}^{i-1}(1-o(\gamma_j))  \middle
\vert\gamma_i\leq \tau_l\right]}{\sum_{i=1}^{\infty} E_{\xi_1,...,\xi_i}\left[ o(\gamma_i)\prod_{j=1}^{i-1}(1-o(\gamma_j))  \middle\vert\gamma_i\geq \tau_l\right]}$.
\eref{eq_fresh} derives from applying  Jensen's integral inequality to the above expression for staleness $F$.
\end{proof}}

\end{appendix}

% (just for communications)
% \begin{align}
% \begin{cases}
% \displaystyle 
% \dfrac{d a(t)}{d t}=\dfrac{2g}{D}a(t)(1-a(t))S(t) -\dfrac{2\alpha}{DR}a(t),
% \end{cases}
% \label{eq:ode}
% \end{align}

%\com{THIS WAS IN THE INTRO -- To be moved elsewhere or deleted}